\documentclass[10pt,twocolumn,letterpaper]{article}

\usepackage[pagenumbers]{iccv} % Optional if you want page numbers

\usepackage{amsthm} % For theorem environments
\usepackage{amssymb} % For mathematical symbols
\usepackage{latexsym} % Additional symbols
\usepackage{amsmath} % For mathematical environments like split

\newtheorem{theorem}{Theorem}
\newtheorem{lemma}[theorem]{Lemma}

\definecolor{iccvblue}{rgb}{0.21,0.49,0.74}
\usepackage[pagebackref,breaklinks,colorlinks,allcolors=iccvblue]{hyperref}

\def\confName{ICCV}
\def\confYear{2025}

\title{Compact Recurrent Transformer with Persistent Memory}

\author{
Edison Mucllari\\
University of Kentucky\\
\and
Zachary Daniels\\
SRI International\\
\and
David Zhang\\
SRI International\\
\and
Qiang Ye\\
University of Kentucky\\
}

\begin{document}
\maketitle
\begin{abstract}
The Transformer architecture has shown significant success in many language processing and visual tasks. However, the method faces challenges in efficiently scaling to long sequences because the self-attention computation is quadratic with respect to the input length. To overcome this limitation, several approaches scale to longer sequences by breaking long sequences into a series of segments, restricting self-attention to local dependencies between tokens within each segment and using a memory mechanism to manage information flow between segments. However, these approached generally introduce additional compute overhead that restricts them from being used for applications where limited compute memory and power are of great concern (such as edge computing). We propose a novel and efficient Compact Recurrent Transformer (CRT), which combines shallow Transformer models that process short local segments with recurrent neural networks to compress and manage a single persistent memory vector that summarizes long-range global information between segments. We evaluate CRT on WordPTB and WikiText-103 for next-token-prediction tasks, as well as on the Toyota Smarthome video dataset for classification. CRT achieves comparable or superior prediction results to full-length Transformers in the language datasets while using significantly shorter segments (half or quarter size) and substantially reduced FLOPs. Our approach also demonstrates state-of-the-art performance on the Toyota Smarthome video dataset.
\end{abstract}    
\section{Introduction}
\label{sec:intro}

The Transformer architecture has marked a significant leap in language modeling \cite{radford2018improving, devlin2018bert, jaegle2021perceiver, chowdhery2023palm, raffel2020exploring, thoppilan2022lamda}, primarily attributed to the introduction of the self-attention mechanism \cite{vaswani2017attention}. This key component allows the model to weigh the importance of different words/tokens in a sequence while processing each word, enabling the ability to dynamically model context and relationships between tokens. The sequential nature of both text and video data has led to similar applications of Transformer architectures in video understanding, where self-attention has been applied on top of convolution layers or even in convolution-free architectures to capture temporal relationships. Despite its success across domains, the standard Transformer architecture is limited to processing fixed-length sequences, hindering its ability to capture longer-term dependencies. In theory, Transformers have the capability to process entire input sequences, with the self-attention mechanism attending to every token simultaneously. However, limitations in terms of memory usage and compute power, make this infeasible due to the quadratic complexity of the self-attention mechanism. Instead, to process long sequences requires dividing the sequence into smaller segments for training and processing each segment independently of the others.

To overcome this limitation, Transformer-XL \cite{dai2019transformer} breaks sequences into segments and attends to not only to the current segment but also to those from previous segment (without considering gradients for the memory component). The Recurrent Memory Transformer (RMT) \cite{bulatov2022recurrent} implements a memory mechanism through the inclusion of special read/write memory tokens to pass information between segments during self-attention. The Block Recurrent Transformer \cite{hutchins2022block} embeds a form of recurrence into the Transformer architecture. Perceiver IO \cite{jaegle2021perceiver} translates an expansive and arbitrary input into a consistent latent representation. The Memory Transformer \cite{burtsev2020memory} maintains the core Transformer model structure while incorporating memory through special memory tokens appended to the input sequence. Memformer \cite{wu2020memformer} introduces a distinct memory module to retain past hidden states in condensed representations.

In the video domain, TimeSformer\cite{bertasius2021space} adapts the Vision Transformer (ViT)\cite{dosovitskiy2020image} to video by extending the self-attention mechanism from the image space to the space-time 3D volume. It processes videos as sequences of patches extracted from individual frames. PI-ViT\cite{reilly2024just} further extends this approach by incorporating additional information through 2D and 3D Skeleton Induction Modules in the Toyota Smarthome video dataset \cite{das2019toyota}.
% with information flowing first through temporal attention and then spatial attention

Our objective is to introduce an enhanced form of memory capable of extracting more information from the past while condensing memory into a single token. Our approach is based around the idea of \textbf{\textit{persistent memory}}: an explicit memory mechanism designed to summarize long sequences as a single state vector. This strategy of utilizing a single vector memory diverges from the common paradigm of using blocks of memory tokens in other recurrent Transformers.

The advantage of the proposed approach is that it allows for compact models that process shorter segments, while exploiting contextual information of longer sequences in a compute-efficient manner. Furthermore, initial experiments suggest that our proposed model can exhibit superior or comparable performance to other recurrent models despite a sizeable reduction in the number of FLOPs. This makes our model well-suited to an expanded range of applications that other Transformer-based models are incapable of addressing due to their compute-heavy operations. Specifically, this includes deployment on low Size, Weight, and Power (low-SWaP) devices, such as those needed for edge computing.

Prior to Transformers, Recurrent Neural Network (RNN) architectures \cite{rumelhart1985learning, hopfield1982neural}, such as GRU \cite{cho2014learning} and LSTM \cite{hochreiter1997long}, were the dominant machine learning-based method in language modeling, while for the video domain, Convolutional Neural Networks (CNNs) were the predominant approach. Compared to Transformers, which operate over fixed-length sequences, RNNs theoretically scale to arbitrary context lengths as they model the continual flow of information and summarize past information as a hidden state vector at each time step. In contrast to existing recurrent Transformer models, which implicitly manage memory as tokens that participate in the self-attention mechanism, RNNs explicitly manage memory using well-studied gating mechanisms, which may allow for more efficient consolidation of information over longer sequences when data is processed sequentially. However, RNNs are limited by difficulties in optimization due to exploding and vanishing gradients, and hidden states have information capacity limits which can become saturated when summarizing long sequences, leading to some information loss. Several works have focused on finding solutions to these problems by incorporating unitary or orthogonal matrices for recurrent weights in RNNs, accompanied by techniques to uphold and maintain these distinctive properties to preserve more memory. Among them is the NCGRU \cite{mucllari2022orthogonal}, which we explore to further improve the capabilities of the persistent memory mechanism.

Motivated by the strengths of both the Transformer and RNN architectures, we propose a novel \textbf{\textit{Compact Recurrent Transformer (CRT)}}. By integrating Transformers with RNNs, our architecture aims to harness improved local modeling of sequences using the self-attention capabilities of Transformers while ensuring the seamless flow of global information through the incorporation of RNN memory. During each iteration, the Transformer processes local segments, and the RNN utilizes the output embeddings of the Transformer to generate a final hidden state, serving as a global memory mechanism for subsequent iterations. A second RNN is used as a learned positional encoding mechanism to integrate the memory token of the previous segment with the tokens of the new segments.

Experimental evaluations were conducted on the Word PTB and Wiki Text-103 datasets for varying sequence lengths and architecture sizes. Our results demonstrate that the proposed CRT outperforms the simple Transformer architecture for the benchmark task of next token prediction as measured by perplexity or bytes per character. For Word PTB and Wiki Text-103, when looking at settings involving small segment lengths, our model even surpasses the performance of Transformer-XL, despite the memory in our CRT being represented as a vector, in contrast to the larger memory structure in Transformer-XL. For larger models/segment sizes, our approach performs about as well as Transformer-XL while requiring significantly fewer FLOPs. Beyond language modeling, our architecture also achieves state-of-the-art results on the Toyota Smarthome video dataset compared to other methods. These findings underscore the efficacy of our proposed architecture in overcoming limitations associated with fixed-context length in both language and video modeling domains.

It is important to note that for identical segment sizes, our CRT is comparable in cost to the standard Transformer with a small added overhead. By incorporating recurrence in both the positional encoding and memory along with the self-attention component, we observe an improved and simplified communication within the model. This combination allows our model, in instances such as Word PTB, to achieve results with a small CRT model comparable to a deeper Transformer-XL model, showing that CRT learned from the data with significantly fewer parameters, as well as state-of-the-art results in Toyota Smarthome.
One drawback of incorporating RNNs is the need of serial processing but since it only affects a small portion of the computations and we are interested in a compact model operating on short segments, this potential reduction in computational efficiency is minor and indeed not observed in our experiments.

Despite sharing similarities with other architectures like RMT and Block Recurrent Transformer, our model has some distinct characteristics. While other approaches operate over a block for memory, ours compresses history into a single memory vector. Additionally, our approach explicitly models memory using a pure RNN model as opposed to implicitly modeling memory, and we extend utilizing the RNN structure to the positional encoding for further improving model performance.

\section{Related Works and Backgrounds}
\label{sec:backgrounds}

\subsection{Related Works}
\label{subsec:relatedwork}

Numerous innovative architectures and algorithms aimed at enhancing various aspects of language modeling have been proposed \cite{bengio2000neural, mikolov2010recurrent, merity2016pointer, gal2016theoretically, grave1609efficient}), but Transformers \cite{vaswani2017attention} have become the dominant approach to language/sequence modeling.

Effective language models require incorporating a broader context; e.g., enriched memory plays a pivotal role in improving predictions. Prior approaches leveraged recurrent neural network (RNN) architectures such as GRU \cite{cho2014learning} and LSTM \cite{hochreiter1997long}. These architectures process data in sequence and maintain history/context information through a hidden state, which is updated every time step. However, training these models using backpropagation-through-time (BPTT) can result in vanishing/exploding gradients during training. Various solutions to this problem have emerged where recurrent weights are updated to maintain orthogonal/unitary attributes (e.g.,  multiplicative updates \cite{Wisdom16}, Givens rotations \cite{jing2019gated}, Householder reflections \cite{Mhammadi16}, and Cayley transforms \cite{helfrich2018orthogonal, madu18, mucllari2022orthogonal, pmlr-v97-lezcano-casado19a,helfrich2020eigenvalue}).

The introduction of Transformers in 2017 brought in new methodologies, showcasing state-of-the-art results \cite{thoppilan2022lamda, chowdhery2023palm, radford2018improving}. However, standard Transformers are restricted to processing a fixed-length context. In response to this challenge, various approaches have been proposed that introduce the concept of memory into Transformers to correlate information from different segments of a longer sequence. Transformer-XL \cite{dai2019transformer} extends the Transformer architecture by incorporating a segment-level recurrence mechanism, allowing the model to capture longer-term dependencies beyond a fixed length. However, this ability comes at a cost: a sizeable increase in the number of tokens used to compute the keys and values of the self-attention mechanism. Due to the quadratic complexity of self-attention, this results in an increase in compute and subsequently negatively effects power consumption and processing latency. Thus, while Transformer-XL typically outperforms the Transformer model in terms of performance, it generally is ill-suited for applications and devices requiring limited resource use.

Recurrent Memory Transformer (RMT) \cite{bulatov2022recurrent} builds upon Transformer-XL as a memory-augmented segment-level recurrent Transformer. Memory allows the model to pass information between segments of the long sequence with the help of recurrence. RMT models memory using sets of read and write tokens in combination with a backpropagation-through-time training mechanism. However, similar to Transformer-XL, computational issues arise, notably increased dimensions in the key and value matrices; albeit, the memory size of RMT is generally smaller than the memory of Transformer-XL.

The Block Recurrent Transformer (BRT) \cite{hutchins2022block} introduced an explicit recurrence model into the Transformer architecture. The BRT introduces a cell that operates on blocks of tokens during training, and is composed of a transformer layer that uses self-attention, cross-attention, and LSTM-style gates to compute a recurrent function over a set of state vectors and tokens. The BRT achieves a lower perplexity than Transformer-XL while running faster and using fewer parameters. This demonstrates the power of properly modeling recurrent relations within the Transformer architecture. We propose an alternative way of doing so that does not require full blocks of memory and instead summarizes memory as a single vector.
% To overcome this limitation, Transformer-XL \cite{dai2019transformer} breaks sequences into segments and attends to not only to the current segment but also to those from previous segment (without considering gradients for the memory component). The Recurrent Memory Transformer (RMT) \cite{bulatov2022recurrent} implements a memory mechanism through the inclusion of special read/write memory tokens to pass information between segments during self-attention. The Block Recurrent Transformer \cite{hutchins2022block} embeds a form of recurrence into the Transformer architecture. Perceiver IO \cite{jaegle2021perceiver} translates an expansive and arbitrary input into a consistent latent representation. The Memory Transformer \cite{burtsev2020memory} maintains the core Transformer model structure while incorporating memory through special memory tokens appended to the input sequence. Memformer \cite{wu2020memformer} introduces a distinct memory module to retain past hidden states in condensed representations.

In video datasets, TimeSformer\cite{bertasius2021space} marked a significant advancement by introducing a convolution-free approach to video transformers. Similar to how language models uses different attention mechanisms, TimeSformer\cite{bertasius2021space} employs divided space-time attention where each layer first attends through the time dimension before attending through the spatial dimension. This architectural choice mirrors the self-attention mechanisms in language modeling that separate different types of contextual information, resulting in improved performance across various video benchmarks. Building upon TimeSformer's framework, PI-ViT\cite{reilly2024just} further enhances performance on the Toyota Smarthome\cite{das2019toyota} dataset by incorporating extra information from 2D and 3D skeleton data.

Our Compact Recurrent Transformer shares similarities with the aforementioned extensions of the Transformer architecture that break long sequences into segments, locally process segments, and propagate memory information between segments. However, there are distinctive differences. Transformer-XL, RMT, and Block Recurrent all rely on processing block-based memory mechanisms and their training does not involve passing gradient between segments. 
%to facilitate information transfer between segments. 
In contrast, CRT uses a single vector-based memory from the recurrent model and is trained by passing gradients through the memory vector to the previous segments by the BPTT algorithm. Additionally, despite integrating recurrence with the Transformer, CRT structurally decouples the global memory RNN from the local Transformer, thus our approach can generally be used with most Transformer-like architectures.

\subsection{Background}
\label{subsec:background}

Language modeling involves the integration of a corpus of tokens, represented as $\textbf{x}=(x_1, ..., x_T)$, where the goal is to estimate the joint probability $P(\textbf{x})$. By using probability theory, this joint probability is factorized as $P(\textbf{x}) = \prod _t P(x_t | x_{x_1, ..., x_{t-1}})$, indicating that it only requires the calculation of each conditional probability. In recent years, neural networks have been employed to model the conditional probability, which involves leveraging information from $x_1$ to $x_{t-1}$ and generating a probability distribution for the subsequent predicted token.

\textbf{Transformer and Self-Attention} \cite{vaswani2017attention}: The Transformer architecture relies on the Multi-Head Self-Attention mechanism. This component is crucial in ensuring that the model pays attention to every input token, enhancing its ability to capture complex patterns in the data.

The self-attention mechanism borrows ideas from information-retrieval, where keys $K$, queries $Q$, and values $Q$ are computed from input tokens $X$, these representations are used to identify related tokens, and information is aggregated between related tokens. For simplicity, we assume that $Q$, $K$ and $V$ represent the outputs of a single attention layer in the Transformer, and they include positional encoding information (information about each token's position in the segment). Multi-headed attention involves learning multiple output embeddings via training different randomly initialized key, query, and value embedding matrices and performing a weighted aggregation.

Consider $d_m$ as the embedding dimension. If the Transformer employs $h$ heads, the Self-Attention  for every head $H_i=\text{Attention}(Q^i, K^i, V^i)$ is expressed by:
\begin{equation} \label{eq1}
\begin{split}
&\text{Attention}(Q^i, K^i, V^i) 
  = \text{softmax}\left(\frac{Q^i  {K^i}^T}{\sqrt{d_k}}\right) V^i \\
&Q^i = X  W_Q^i, \ \ K^i = X  W_K^i, \ \ V^i = X  W_V^i
\end{split}
%\begin{split}
%H_i & = \text{Attention}(Q^i, K^i, V^i) \\
% & = \text{softmax}\left(\frac{Q^i  {K^i}^T}{\sqrt{d_k}}\right) V^i
%\end{split}
\end{equation}
%$$Q^i = X  W_Q^i, \ \ K^i = X  W_K^i, \ \ V^i = X  W_V^i$$
Here $Q, K, V \in \mathbb{R}^{n \times d_m}$, $W_Q^i, W_K^i, W_V^i \in \mathbb{R}^{d_m \times d_k}$, where 
$Q, K, V$ are input embedding matrices, $W_Q^i, W_K^i, W_V^i$ are learnable parameters and $n$ is the segment length, $d_m$ is the embedding dimension, $h$ is the number of heads and $d_k$ is the dimension of a single head, which we set to be $\frac{d_m}{h}$.

The computation for the Self-Attention component is independently performed for all $h$ heads. The $\text{MultiHead(Q, K, V)}$ is then obtained by concatenating the results $H_1, ..., H_h$ and applying a weight matrix $W_O$:
$$\text{MultiHead(Q, K, V)} = \text{Concat}(H_1, ..., H_h) * W_O$$
where $W^O \in \mathbb{R}^{d_m \times d_m}$. The Multi-Head Self-Attention mechanism plays a key role in improving the model's capability to understand relationships within the input data.

\section{Compact Recurrent Transformer}
\label{sec:crt}

We introduce a new architecture that integrates Transformer and RNN architectures to address the limitations of fixed contextual learning in the standard Transformer model. While the Transformer architecture has demonstrated exceptional performance across various applications, including Natural Language Processing, when context length is fixed (i.e., all tokens of a sequence can be processed in parallel), it faces challenges in leveraging historical information when sequences exceed the context length and must be processed in segments. In contrast, RNNs excel in preserving and passing information through its hidden states, allowing them to capture essential context from previous iterations. The RNN module in our experiments consists of GRU \cite{cho2014learning} and NCGRU \cite{mucllari2022orthogonal}.

\begin{figure}[ht]
\vskip -0.1in
\begin{center}
\centerline{\includegraphics[width=0.9\columnwidth]{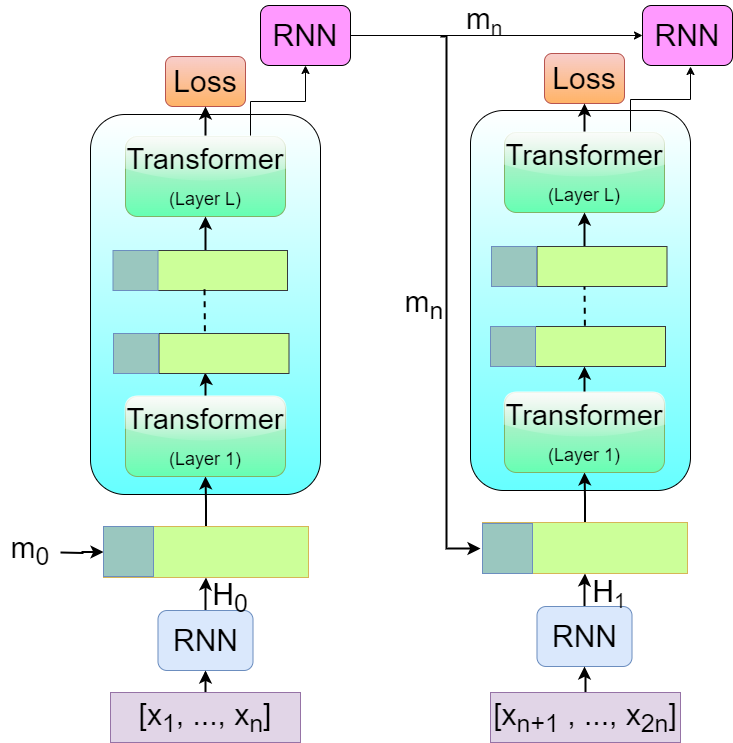}}
\caption{Compact Recurrent Transformer Architecture}
\label{srt_arch}
\end{center}
\vskip -0.2in
\end{figure}

As illustrated in Figure \ref{srt_arch}, our proposed model leverages the strengths of both architectures: recurrence and self-attention. During each iteration, a segment input goes through every layer of the Transformer model (with L denoting the number of Transformer layers, set to three or sixteen in our experiments). Following the final Transformer layer, the output undergoes processing through a linear layer and softmax to compute the loss function for training. Simultaneously, the output tensor enters an GRU/NCGRU model, where the last hidden state from the GRU/NCGRU serves as the memory for that iteration and is subsequently passed to the next iteration. We describe in more details below.

\textbf{Transformer with Memory Token}: As depicted in Figure \ref{srt_mem}, we incorporate a memory unit represented as a single token into a standard Transformer by concatenating it with the input segment. In each attention layer, the model attends not only to the input segment but also to the single vector memory token. Following each attention layer, the memory token undergoes updates from Transformer layer as in other tokens, except in the last layer where no output is generated from the memory token position. 

\begin{figure}[ht]
\begin{center}
\centerline{\includegraphics[width=0.9\columnwidth]{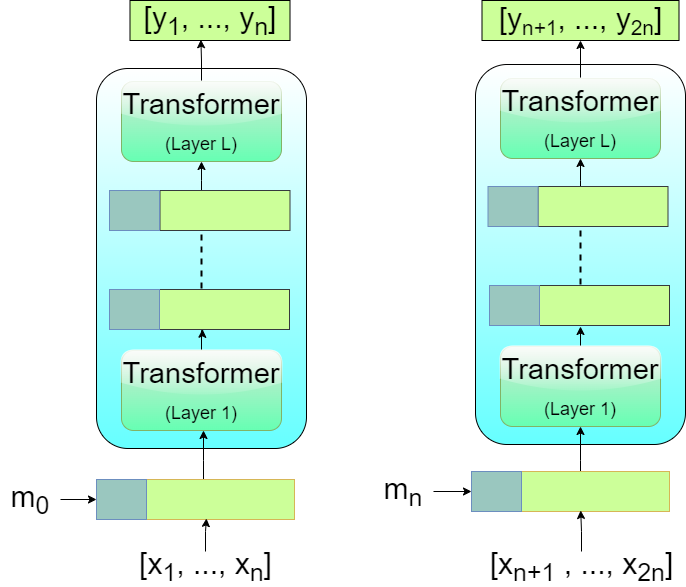}}
\caption{Transformer with Memory Token}
\label{srt_mem}
\end{center}
\vskip -0.2in
\end{figure}

\textbf{RNNs for Persistent Memory}: Figure \ref{rec_mem} depicts how the memory is generated  in our CRT architecture. The output of the Transformer model $[y_1, ..., y_n]$ is fed into the the memory RNN to produce hidden states  $[h_1, ..., h_n]$. The last hidden state $h_n$, rewritten as $m_n$, serves as the memory to be used in the next segment. Simultaneously, $m_n$ acts as the initial hidden state for the next iteration of the memory RNN. Extending this concept, the last hidden state for the current iteration, $h_{2n} = m_{2n}$, becomes the memory for the next iteration and serves as the initial hidden state for the subsequent memory RNN, creating a continuous chain. This demonstrates that the memory, as it passes through iterations, encompasses information in theory from all previous iterations.

\begin{figure}[ht]
\begin{center}
\centerline{\includegraphics[width=\columnwidth]{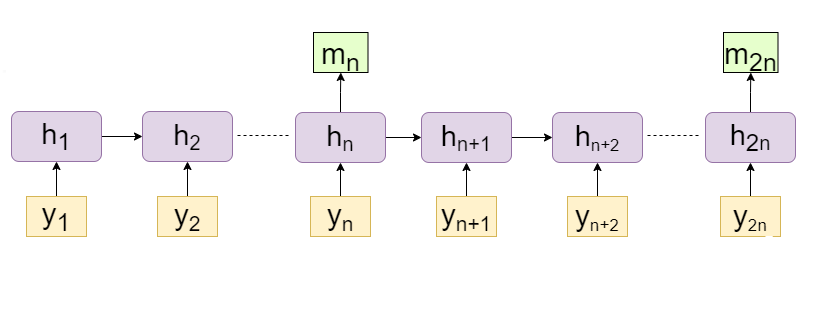}}
\caption{Recurrent Memory Architecture}
\label{rec_mem}
\end{center}
\vskip -0.2in
\end{figure}

\textbf{Recurrent Position Encoding:} A crucial aspect of our architecture is the Recurrent Position Encoding, designed to handle the positional encoding when processing information from the new input sequence. We introduce a new GRU/NCGRU model for encoding the relative position of a token in a segment of a longer sequence. Past work \cite{wang2019r} has shown that RNNs can serve as effective mechanisms for positional encoding, and such a positional encoding scheme has additional benefits for the CRT model. By utilizing a separate RNN model to derive the positional encoding for the new input segment, the CRT model enhances its comprehension and aids in ensuring the memory token and input embeddings are compatible before being processed by the self-attention mechanism. An additional benefit of using RNN position encoding is that the attention score is determined using the simplified expression:
$$A_{i, j}^{CRT} = E_{x_i}^T W_q^T W_{k} E_{x_j} $$
This formulation avoids additional terms involving relative positional encoding (i.e., as is done in Transformer-XL), allowing our model to achieve simplicity and eliminate the computations associated with relative position encoding without compromising effectiveness.

\begin{figure}[ht]
\begin{center}
\centerline{\includegraphics[width=\columnwidth]{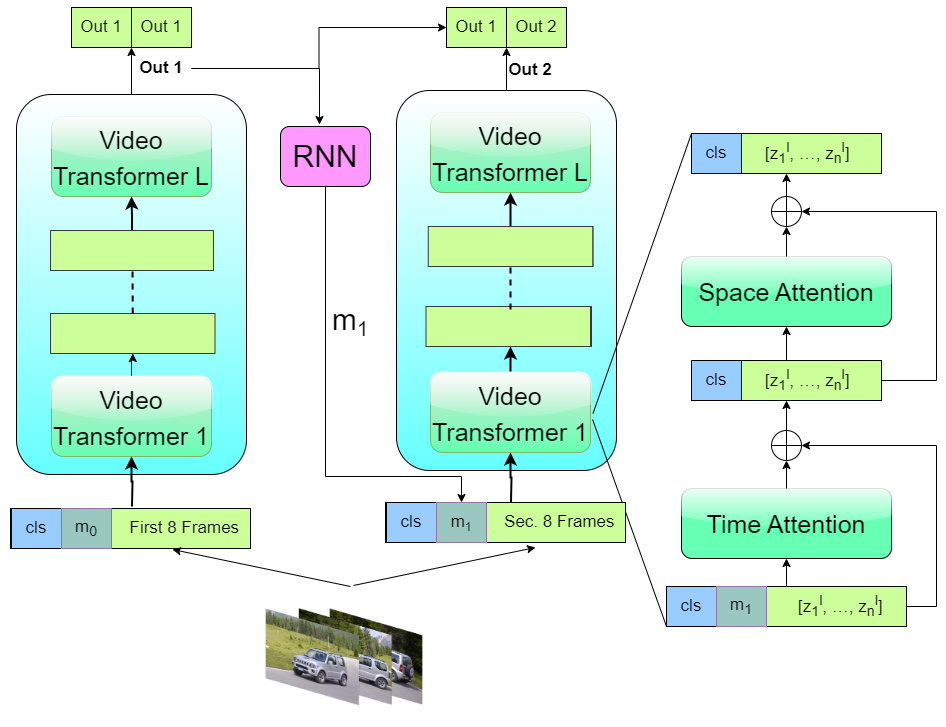}}
\caption{Compact Recurrent ViT Architecture}
\label{cr_vit}
\end{center}
\vskip -0.2in
\end{figure}

\textbf{Compact Recurrent ViT Architecture:} The Compact Recurrent ViT model employs a unique two-iteration training approach for each video. In the first iteration, the memory vector is initialized to zeros. After processing through the network, the output, except the class token, is extracted and passed through a Recurrent Neural Network (RNN) to generate the memory vector for the subsequent iteration. As illustrated in Figure \ref{cr_vit}, the memory token is exclusively utilized during the Time Attention layer. While other architectural components draw inspiration from the TimeSformer model, the specific memory token mechanism with RNN is an original contribution of this approach. The method focuses on enhancing inter-frame information processing capabilities through this novel token strategy. To maximize information extraction across both iterations, the model implements a specific token strategy. During the first iteration, the class token is duplicated, with the input to the final classification layer being [cls token, cls token]. In the second iteration, the input becomes [prev cls token, cls token], where prev cls token represents the class token from the previous iteration's output. A critical information-passing mechanism is also employed: the class token from the second iteration is added to the output class token from the previous iteration before traversing the Transformer layers. This approach enables more comprehensive information transfer between iterations, enriching the model's understanding of the video content.

% To summarize: 1) An embedded input segment of a sequence is passed through the positional-encoding GRU/NCGRU one-by-one. 2) the hidden states of these tokens become the new positional-encoded token embeddings and are combined with the memory token (hidden state of the memory GRU/NCGRU) via concatenation. 3) The updated segment is fed into the Transformer model. 4) The transformer outputs probability distributions for the next token prediction task, and its output embeddings are fed into the persistent memory RNN. 5) The persistent memory RNN outputs a memory token, which is used to process the next sequence segment. 6) The process repeats until the entire sequence is processed.

% You must include your signed IEEE copyright release form when you submit your finished paper.
% We MUST have this form before your paper can be published in the proceedings.

% Please direct any questions to the production editor in charge of these proceedings at the IEEE Computer Society Press:
% \url{https://www.computer.org/about/contact}.
\section{Complexity and Gradient Analysis}
\label{sec:derivatives}

We give flop counts for one forward propagation of each of the Transformer, XL, and CRT-GRU models in Table 1. Ours is comparable to Transformer that has additional cost related to the relative positional encoding, and less than XL. Also listed is the parameter counts, where ours is also less owing to the parameters associated with relative positional encoding. A significant advantage of CRT is in the memory cost, where we need one vector for memory vs $n$ vector memory for XL.

Gradient behavior determines the stability and quality of how models learn. To analyze memory flow across iterations, we investigate the derivative of the current iteration's output, denoted as $y_{k}$, where $tn+1 \leq k \leq (t+1)n$, with respect to the Transformer output $y_i$ of one of the previous iterations, where $pn+1 \leq i \leq (p+1)n$ and $p<t$. In this context, $k$ and $i$ denote the token positions in the entire sequence, while $n$ signifies the segment length. Then, this gradient indicates the information flow from $y_i$ to $y_k$. 

We can prove (see Appendix for details) that 
\begin{equation}
    \label{eq:boundequation}
    \begin{split}
        \frac{\partial y_k}{\partial y_i} & = \frac{\partial y_k}{\partial m_{tn}}  \frac{\partial h_{tn}}{\partial h_{i}}  \frac{\partial h_i}{\partial y_i} \\ 
     \left\| \frac{\partial y_k}{\partial y_i} \right\|_2   & \leq   (\alpha + \beta\|U_c\|_2)^{tn-i}  
     \left\|\frac{\partial y_k}{\partial m_{tn}}  \right\|_2        
     \left\|  \frac{\partial h_j}{\partial y_i} \right\|_2          
    \end{split}
\end{equation}
where $U_c$ and $\alpha, \beta$ are defined in Lemma A.1 in Appendix.
% $$\frac{\partial y_k}{\partial y_i} = \frac{\partial y_k}{\partial m_t} \left(\prod _{j=p+1}^{t-1} \frac{\partial m_{j+1}}{\partial m_{j}}\right) \frac{m_p}{\partial y_i}$$
In particular, in the case that NCGRU is used, when $u_t$ and $r_t$ are approximately either the zero vector or the vector of all ones, the bound becomes:
\begin{align}
     \left\| \frac{\partial y_k}{\partial y_i} \right\|_2   & \leq    
     \left\|\frac{\partial y_k}{\partial m_{tn}}  \right\|_2        
     \left\|  \frac{\partial h_j}{\partial y_i} \right\|_2     
\end{align}

With $y_k$ and $y_i$ representing the tokens at positions $k$ and $i$ resp.,
(\ref{eq:boundequation}) signifies the information flow from token $i$ to GRU state $h_i$, which is transmitted to GRU state $h_{tn}$, and finally to $y_j$ by the $t$-th Transformer through the memory $m_{2n}$. In particular, in the case of NCGRU, we have an improved bound on $\frac{\partial h_{tn}}{\partial h_{i}}$, which allows passing memory across different segments.

\begin{table*}[t]
\caption{Complexity Comparison.  $L$ is the number of layers, $n$ is the segment length, and $d_m$ is the embedding dimension.}
\label{complexity}
\begin{center}
\begin{small}
\begin{sc}
\begin{tabular}{lccr}
\toprule
Model &  FLOPs & \# Params \\
\midrule
Tr     & $L(20nd_m^2+6n^2d_m) + O(nd_m)$  & $L(10d_m^2+6d_m)$ \\
Tr-XL  & $L(28nd_m^2+12n^2d_m + 6n^2)+ O(nd_m)$ & $L(10d_m^2+6d_m)$ \\
CRT-GRU   & $nd_m^2(12L+12)+8Ln^2d_m + O(nd_m)$ & $(12+8L)d_m^2+(6+4L)d_m$ \\
\bottomrule
\end{tabular}
\end{sc}
\end{small}
\end{center}
\vskip -0.1in
\end{table*}
\section{Experiments}
\label{sec:experiments}

We evaluate the effectiveness of the Compact Recurrent Transformer (CRT) by conducting comprehensive experiments across diverse datasets: Word PTB \cite{marcus1993building} and WikiText-103 \cite{merity2016pointer} for language modeling, and Toyota Smarthome for video recognition. For language modeling, we compare CRT against Transformer-XL and standard Transformer architectures as baselines, exploring three-layer models for all datasets and sixteen-layer models for Word PTB and WikiText-103, using either GRU or NCGRU for the RNN component. In video recognition tasks, we employ CRT (CR-ViT) with GRU as the RNN component, comparing it with TimeSformer and other SOTA models on the Toyota Smarthome dataset such as PI-ViT. In all tables, segment length refers to the number of tokens in each segment and memory length refers to the number of vectors used as memory. This multi-domain evaluation demonstrates CRT's versatility and performance.

\begin{figure}[ht]
\begin{center}
\centerline{\includegraphics[width=\columnwidth]{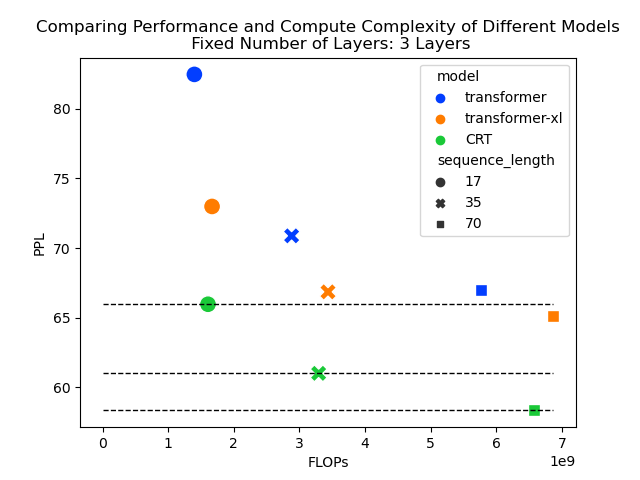}}
\caption{Comparing the CRT model to the baseline Transformer and Transformer-XL for a 3-layer model in terms of PPL and FLOPs on different segment sizes for the Word PDB dataset. Given the same segment length, the CRT outperforms the baseline Transformer and Transformer-XL models. An interesting observation is that The CRT model using the smaller segment size of 17 tokens performs about as well as the Transformer and Transformer-XL models with segment sizes of 70 tokens.}
\label{fig:performance_v_flops_3_layer}
\end{center}
\vskip -0.2in
\end{figure}

\begin{figure}[ht]
\begin{center}
\centerline{\includegraphics[width=\columnwidth]{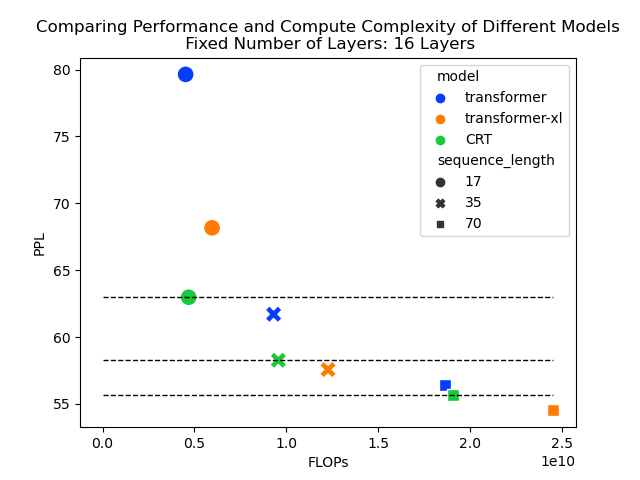}}
\caption{Comparing the CRT model to the baseline Transformer and Transformer-XL for a 16-layer model in terms of PPL and FLOPs on different segment sizes for the Word PDB dataset. Note that in this case, the CRT model always outperforms the baseline Transformer model at the same segment size and is generally close to the performance of the baseline Transformer model with 2x segment size. At the same segment size, CRT is competitive with Transformer-XL, but requires noticeably fewer FLOPs.}
\label{fig:performance_v_flops_16_layer}
\end{center}
\vskip -0.4in
\end{figure}

\subsection{Word PTB}
\label{subsec:wordptb}

The Word PTB dataset serves as a valuable benchmark, comprising English-language Wall Street Journal articles. The dataset is partitioned into training, validation, and test sets, containing of 888 thousand, 70 thousand, and 79 thousand words, respectively and a vocabulary of 10 thousand words.

\begin{table}[ht]
% \caption{Word PTB Perplexity (PPL) results on 3-layer models. Three-layer CRT-NCGRU processing segments of 70 tokens achieves a perplexity of 58.3 while the sixteen layer Transformer also processing segments of 70 tokens achieves a perplexity of 56.4}
\caption{Word PTB Perplexity (PPL) results on 3-layer models. Three-layer CRT processing segments of 70 tokens achieves a perplexity of 58.3 while the sixteen layer Transformer also processing segments of 70 tokens achieves a perplexity of 56.4}
\label{wptbthree}
\begin{center}
\begin{small}
\begin{sc}
\begin{tabular}{lccr}
\toprule
Model & Memory/Seg Len & PPL \\
\midrule
Transformer    & 0 / 70 & 67.0 \\
Transformer-XL & 70 / 70 & 65.1 \\
% CRT - GRU \textbf{(ours)} & 1 / 70 & 59.4 \\
% CRT - NCGRU \textbf{(ours)}  & 1 / 70 & 58.3 \\
CRT \textbf{(ours)}  & 1 / 70 & 58.3 \\
\midrule
Transformer    & 0 / 35 & 70.9 \\
Transformer-XL & 35 / 35 & 66.9 \\
% CRT - GRU \textbf{(ours)} & 1 /  35 & 61.0 \\
% CRT - NCGRU \textbf{(ours)} & 1 /  35 & 61.0 \\
CRT \textbf{(ours)} & 1 /  35 & 61.0 \\
\midrule
Transformer    & 0 /  17 & 82.5 \\
Transformer-XL & 17 /  17 & 73.0 \\
% CRT - GRU \textbf{(ours)}  & 1 /  17 & 65.6 \\
% CRT - NCGRU \textbf{(ours)} & 1 / 17 & 66.0 \\
CRT \textbf{(ours)} & 1 / 17 & 66.0 \\
\bottomrule
\end{tabular}
\end{sc}
\end{small}
\end{center}
\vskip -0.1in
\end{table}

Tables \ref{wptbthree} and \ref{wptbsixteen} present the results obtained using NCGRU as the recurrence mechanism, with additional comparative analysis against GRU provided in Section \ref{rec_exp_sup}. These tables highlight that CRT consistently outperforms both Transformer and Transformer-XL models across a range of configurations.
% In Tables \ref{wptbthree} and \ref{wptbsixteen}, we demonstrate that CRT outperforms both Transformer and Transformer-XL models across various configurations. 
For both shallow (3-layer) and deep (16-layer) architectures, CRT achieves superior or comparable perplexity results while processing significantly shorter segments. 

For the three-layer model, when segment length is set to 70 (standard setting per Transformer-XL \cite{dai2019transformer} results on Word PTB), CRT processes smaller segments (17 tokens) and surpasses the Transformer's performance with 4x longer sequence lengths. Unlike Transformer-XL, which doubles the number of keys and values to extend memory, CRT adds only a single memory token, effectively increasing the segment length by just one. In the 16-layer configuration with 70-token segments, this results in a reduction of $5e9$ FLOPs compared to Transformer-XL.

A key finding is that CRT maintains performance as segment length decreases. The 16-layer CRT achieves 55.7 perplexity with 70-token segments and only drops to 63.0 with 17-token segments. In contrast, the standard Transformer's performance deteriorates dramatically from 56.4 to 79.6 perplexity when segment length is reduced from 70 to 17 tokens. This demonstrates the effectiveness of our persistent memory mechanism—the RNN efficiently compresses long sequences and manages information propagation, conserving computational resources in the self-attention mechanism while preserving context between segments.

\begin{table}[ht]
\caption{Word PTB Perplexity (PPL) results on 16-layer models. Results for Transformer-XL with 70 tokens from \cite{dai2019transformer}.}
\label{wptbsixteen}
\begin{center}
\begin{small}
\begin{sc}
\begin{tabular}{lccr}
\toprule
Model & Memory  / Seg Len & PPL \\
\midrule
Transformer    & 0 / 70 & 56.4 \\
Transformer-XL & 70 / 70 & 54.5 \\
% CRT - GRU \textbf{(ours)} & 1 / 70 & 56.5 \\
% CRT - NCGRU \textbf{(ours)}   & 1 / 70 & 55.7 \\
CRT \textbf{(ours)}   & 1 / 70 & 55.7 \\
\midrule
Transformer    & 0 / 35 & 61.7 \\
Transformer-XL & 35 / 35 & 57.6 \\
% CRT - GRU \textbf{(ours)} & 1 / 35 & 58.2 \\
% CRT - NCGRU \textbf{(ours)}   & 1 / 35 & 58.3 \\
CRT \textbf{(ours)}   & 1 / 35 & 58.3 \\
\midrule
Transformer    & 0 / 17 & 79.6 \\
Transformer-XL & 17 / 17 & 68.2 \\
% CRT - GRU \textbf{(ours)} & 1 / 17 & 64.2 \\
% CRT - NCGRU \textbf{(ours)}   & 1 / 17 & 63.0 \\
CRT \textbf{(ours)}   & 1 / 17 & 63.0 \\
\bottomrule
\end{tabular}
\end{sc}
\end{small}
\end{center}
\vskip -0.3in
\end{table}

\subsection{WikiText-103}
\label{subsec:wikitext103}

% The WikiText-103 dataset is a large word-level language modeling benchmarks with a focus on long-term dependency. This dataset boasts 103 million training tokens extracted from 28,000 articles, each having an average length of 3600 tokens. This enables evaluation of a model's ability to handle long-term dependencies under varied content.
The WikiText-103 dataset is a large word-level language modeling benchmark that focuses on long-term dependency. It consists of 103 million training tokens extracted from 28,000 articles, each with an average length of 3600 tokens. This dataset allows for the evaluation of a model's ability to handle long-term dependencies across diverse content. Notably, the standard sequence length for WikiText-103 is typically set to 150 tokens, but we also experiment with shorter sequence lengths. This is because our primary goal is to develop a model that can be efficiently deployed on smaller edge devices, where computational resources are limited.

% The experimental results shown in Tables \ref{wikitext103three} and \ref{wikitext103sixteen} consistently aligns with those observed on the Word PTB dataset, showing the superiority of our proposed CRT model over the baseline Transformer architecture. 
Similarly to the experiments conducted on the Word PTB dataset, Tables \ref{wikitext103three} and \ref{wikitext103sixteen} present the results obtained using the NCGRU recurrence model, along with a detailed comparative analysis between NCGRU and GRU in Section \ref{rec_exp_sup}. The experimental outcomes in Tables \ref{wikitext103three} and \ref{wikitext103sixteen} consistently mirror those observed on the Word PTB dataset, highlighting the superiority of our proposed CRT model over the baseline Transformer architecture.
Furthermore, our Compact Recurrent Transformer performs near (and sometimes better than) the Transformer-XL model, achieving perplexities within one to two points of the Transformer-XL for all settings. This is noteworthy considering that the number of keys and values used in the self-attention mechanism for the Transformer-XL models approximately doubles compared to the baseline Transformer whereas for the CRT, the segment length is extended only by one by compressing memory into just a single vector from the RNN hidden state.

\begin{table}[ht]
\caption{Wiki Text-103 PPL results on 3-layer models}
\label{wikitext103three}
\begin{center}
\begin{small}
\begin{sc}
\begin{tabular}{lccr}
\toprule
Model & Memory / Seg Len & PPL \\
\midrule
Transformer    & 0 / 150  & 39.1 \\
Transformer-XL & 150 / 150 & 32.6 \\
% CRT - GRU \textbf{(ours)}  & 1 / 150 & 32.3 \\
% CRT - NCGRU \textbf{(ours)}   & 1 / 150 & 31.8 \\
CRT \textbf{(ours)}   & 1 / 150 & 31.8 \\
\midrule
Transformer    & 0 / 70  & 47.3 \\
Transformer-XL & 70 / 70 & 36.5 \\
% CRT - GRU \textbf{(ours)}  & 1 / 70 & 36.6 \\
% CRT - NCGRU \textbf{(ours)}   & 1 / 70 & 35.8 \\
CRT \textbf{(ours)}   & 1 / 70 & 35.8 \\
\midrule
Transformer    & 0 / 35  & 62.4 \\
Transformer-XL & 35 / 35 & 43.9 \\
% CRT - GRU \textbf{(ours)}  & 1 / 35 & 43.0 \\
% CRT - NCGRU \textbf{(ours)}   & 1 / 35 & 43.0 \\
CRT \textbf{(ours)}   & 1 / 35 & 43.0 \\
\bottomrule
\end{tabular}
\end{sc}
\end{small}
\end{center}
\vskip -0.1 in
\end{table}

\begin{table}[ht]
\caption{Wiki Text-103 PPL results on 16-layer models. Results for Transformer-XL with 150 tokens is from \cite{dai2019transformer}.}
\label{wikitext103sixteen}
\begin{center}
\begin{small}
\begin{sc}
\begin{tabular}{lccr}
\toprule
Model & Memory / Seg Len & PPL \\
\midrule
Transformer    & 0 / 150  & 30.0 \\
Transformer-XL & 150 / 150 & 24.0 \\
% CRT - GRU \textbf{(ours)}  & 1 / 150 & 25.4  \\
% CRT - NCGRU \textbf{(ours)} & 1 / 150 & 25.8 \\
CRT \textbf{(ours)} & 1 / 150 & 25.8 \\
\midrule
Transformer    & 0 / 70  & 37.4 \\
Transformer-XL & 70 / 70 & 27.6 \\
% CRT - GRU \textbf{(ours)}  & 1 / 70 & 28.6  \\
% CRT - NCGRU \textbf{(ours)} & 1 / 70 & 28.8 \\
CRT \textbf{(ours)} & 1 / 70 & 28.8 \\
\midrule
Transformer    & 0 / 35  & 51.9 \\
Transformer-XL & 35 / 35 & 34.2 \\
% CRT - GRU \textbf{(ours)}  & 1 / 35 & 35.5  \\
% CRT - NCGRU \textbf{(ours)} & 1 / 35 & 35.8 \\
CRT \textbf{(ours)} & 1 / 35 & 35.8 \\
\bottomrule
\end{tabular}
\end{sc}
\end{small}
\end{center}
\vskip -0.2in
\end{table}

Table \ref{wikitext103complexity} shows that the inference times for CRT and Transformer are comparable, whereas Transformer-XL's inference time is nearly 50\% higher. 
%400 ms/batch higher. 
While one might expect the RNNs of CRT to increase processing times due to processing data in sequence instead of in parallel, the RNNs used in our experiments are light-weight, so their compute cost is dominated by the cost of the Transformer. Additionally, our approach requires fewer parameters compared to both Transformer and Transformer-XL, yet it outperforms the baseline Transformer and is competitive with Transformer-XL across various segment lengths.

\begin{table}[ht]
\caption{Wiki Text-103 PPL number of parameters and Inference time comparison on 16-layer models}
\label{wikitext103complexity}
\begin{center}
\begin{small}
\begin{sc}
\begin{tabular}{lccr}
\toprule
Model & \# Params & Inf. Time (ms/batch) \\
\midrule
Transformer    & 151M  & 997 \\
Transformer-XL & 151M  & 1421.3 \\
CRT - GRU \textbf{(ours)}  & 150M & 997.5  \\
\bottomrule
\end{tabular}
\end{sc}
\end{small}
\end{center}
\vskip -0.2in
\end{table}

\subsection{Toyota Smarthome}
\label{subsec:toyota}

Toyota Smarthome\cite{das2019toyota} is a large real-world video dataset capturing activities of daily living (ADL). It contains 16K RGB+D clips spanning 31 activity classes performed by seniors in a smart home environment. The dataset features both fine-grained activities (e.g., drinking from a cup) and composite activities like cooking. Recordings were made across 3 different scenes using 7 cameras. Importantly, subjects received no instructions on how to perform these activities, resulting in natural behaviors that closely resemble real-world scenarios, making the dataset particularly valuable for practical applications. The primary task associated with this dataset is a classification problem where models must correctly identify which of the 31 predefined activity classes is being performed in each video clip.

\begin{table}[ht]
\caption{Comparison with state-of-the-art methods on Toyota Smarthome dataset. Mean class accuracy (mCA) reported for cross-subject (CS) evaluation.}
% \label{enwikithree}
\begin{center}
\begin{small}
\begin{sc}
\begin{tabular}{lccr}
\toprule
Model & mCA \\
\midrule
2s-AGCN & 60.9 \\
PoseC3D & 50.6 \\
Hyperformer & 57.5 \\
VPN & 65.2 \\
AssembleNet++ & 63.6 \\
UNIK & 64.6 \\
Video Swin & 69.8 \\
MotionFormer & 65.8 \\
MMNET & 70.1 \\
TimeSformer  & 71.5$^*$ \\
% \textbf{CRViT - GRU} \textbf{(ours)}  & \textbf{73.4} \\
\textbf{CRViT} \textbf{(ours)}  & \textbf{73.4} \\
\midrule
TimeSformer    & 68.4 \\
\hspace{0.2cm} +2D-SIM & 72.5 \\
\hspace{0.2cm} +3D-SIM & 71.4 \\
PI-ViT (2D-SIM + 3D-SIM) & 72.9 \\
\bottomrule
\end{tabular}
\end{sc}
\end{small}
\end{center}
\vskip -0.1in
\end{table}

Our Compact Recurrent Vision Transformer (CR-ViT) model, which employs a Gated Recurrent Unit (GRU) as its recurrent component, achieves state-of-the-art (SOTA) performance on the Toyota Smarthome dataset. Notably, the model accomplishes this breakthrough without relying on additional 2D and 3D pose information that other models may utilize. The model's development involved a strategic two-stage training approach. First, we pretrained the TimeSformer architecture on Toyota Smarthome data, using a methodology similar to PI-ViT but omitting the 2D and 3D Skeleton Induction Module (SIM). The results are recorded at the Table with the $^*$. We initialized the model with the Kinetics 400 pretrained weights from TimeSformer and conducted an initial training phase of 15 epochs. Subsequently, we integrated the changes for CR-ViT and train for 40 epochs. This extended training was crucial, as the model demonstrated continued learning beyond the 15-epochs. The results are significant: CR-ViT outperforms PI-ViT by achieving mCA of 73.5, compared to PI-ViT's 72.9, using only the video dataset and corresponding labels. This superior performance was realized without incorporating supplementary pose information, highlighting the model's innovative design. 
%We hypothesize that the model's performance could be further improved with longer video sequences. The current limitations of the Toyota Smarthome dataset, which restricts the analysis to two iterations per video due to their shorter duration, may be constraining the model's full potential. 
Our experiments demonstrate that combining Transformer architecture with RNN components in the time dimension can significantly enhance video understanding and classification performance.

\subsection{Ablation Studies}
\label{subsec:ablation}

\begin{table}[h]
\caption{Ablation studies on Word PTB using a 3-layer model and segment size of 70 tokens}
\label{ablation}
\begin{center}
\begin{small}
\begin{sc}
\begin{tabular}{lcr}
\toprule
Model & PPL \\
\midrule
Transformer    & 67.0  \\
%Transformer-XL & 65.1 \\
Transformer + RNN memory  & 64.4  \\
Transformer + RNN pos enc & 61.0 \\
CRT (Trans. + RNN memory + RNN pos enc) & 58.3 \\
\bottomrule
\end{tabular}
\end{sc}
\end{small}
\end{center}
\vskip -0.2in
\end{table}

We also conducted an ablation study to investigate  the different components of our proposed approach. We compare the baseline Transformer with three models: Transformer with RNN-based memory only, Tranformer with RNN-based positional encoding only, and the full CRT model. We report the results on the word PTB dataset in Table \ref{ablation} where we see steady improvement as additional components are added. It is interesting to note the strong effect of adding an RNN-based positional encoding, and we also see that the RNN memory and the RNN positional encoding have complimentary effects.
\section{Conclusion}
\label{sec:conclusion}

We have demonstrated how persistent memory can be added to Transformers using an explicit RNN-based mechanism. This mechanism allows for compact models that process long sequences by operating over shorter segments. This significantly reduces computational complexity compared to other recurrent Transformer models while achieving equal or better performance, and the proposed approach exhibits notably higher performance compared to running inference using a baseline fixed-length Transformer operating over independent segments. These capabilities open new opportunities for running light-weight, but powerful models on low-SWaP devices. For example, we envision CRT could be useful in future applications where Transformer-based models are run on distributed edge devices instead of large server farms of high-powered but inefficient GPUs. Our experimental results suggest that during inference, each edge device could run shallow models that process short segments of longer sequences and still achieve results comparable to deep Transformer models that process very long sequences. In the extreme case involving models trained with context lengths of hundreds of thousands of tokens at a time (e.g., GPT4 \cite{achiam2023gpt}), the savings in terms of power and latency during inference would be very significant.
\section{Acknowledgement}
\label{acknowledgement}
This research is based upon work supported in part by NSF under the grants  IIS-2327113, DMS-2208314, and ITE-2433190, and by the Office of the Director of National Intelligence (ODNI), Intelligence Advanced Research Projects Activity (IARPA), via Contract No: 2022-21100600001. The views and conclusions contained herein are those of the authors and should not be interpreted as necessarily representing the official policies, either expressed or implied, of ODNI, IARPA, or the U.S. Government. The U.S. Government is authorized to reproduce and distribute reprints for governmental purposes notwithstanding any copyright annotation therein.

We would like to thank the University of Kentucky Center for Computational Sciences and Information Technology Services Research Computing for their support and use of the Lipscomb Compute Cluster and associated research computing resources. We also extend our appreciation to Dominick Reilly for his valuable help in our work on the Toyota Smarthome Dataset.
{
    \small
    \bibliographystyle{ieeenat_fullname}
    \bibliography{main}
}
\appendix
\onecolumn

\begin{center}
    \rule{\textwidth}{4pt}
    \vspace{4pt}
    
    {\Large\textbf{Compact Recurrent Transformer with Persistent Memory}} 
    
    {\Large\textbf{Appendix}}

    \vspace{4pt}
    \rule{\textwidth}{1pt}
\end{center}

\setcounter{section}{0}
\renewcommand\thesection{\Alph{section}}
\renewcommand\thesubsection{\thesection.\arabic{subsection}}

\section{Additional Recurrence Analysis from Experiments}\label{rec_exp_sup}

In this section, we will examine an analysis of the differences observed when using GRU versus NCGRU in our experiments conducted on the Word PTB, Wiki Text-103, and Toyota Smarthome datasets.

\subsection{Word PTB}\label{subsec:sup_wordptb}

\begin{table}[ht]
\caption{Word-level PTB Perplexity (PPL) results for 3-layer models using GRU and NCGRU.}
\label{wptbthree_sup}
\begin{center}
\begin{small}
\begin{sc}
\begin{tabular}{lccr}
\toprule
Model & Memory/Seg Len & PPL \\
\midrule
% Transformer    & 0 / 70 & 67.0 \\
% Transformer-XL & 70 / 70 & 65.1 \\
CRT - GRU & 1 / 70 & 59.4 \\
CRT - NCGRU  & 1 / 70 & 58.3 \\
\midrule
% Transformer    & 0 / 35 & 70.9 \\
% Transformer-XL & 35 / 35 & 66.9 \\
CRT - GRU & 1 /  35 & 61.0 \\
CRT - NCGRU & 1 /  35 & 61.0 \\
\midrule
% Transformer    & 0 /  17 & 82.5 \\
% Transformer-XL & 17 /  17 & 73.0 \\
CRT - GRU  & 1 /  17 & 65.6 \\
CRT - NCGRU & 1 / 17 & 66.0 \\
\bottomrule
\end{tabular}
\end{sc}
\end{small}
\end{center}
\vskip -0.1in
\end{table}

In Table \ref{wptbthree_sup}, we observe that replacing GRU with NCGRU in the recurrence part of our proposed architecture improves performance when the sequence length in Word PTB is set to 70. For a sequence length of 35, the performance remains identical, while for a sequence length of 17, GRU slightly outperforms NCGRU. Similarly, in Table \ref{wptbsixteen_sup}, which presents results for models with 16 layers, NCGRU outperforms GRU in two experiments. When comparing sequence lengths, NCGRU and GRU demonstrate very similar performance, with perplexity scores of 58.3 and 58.2, respectively.

\begin{table}[ht]
\caption{Word PTB Perplexity (PPL) results on 16-layer models using NCGRU and GRU.}
\label{wptbsixteen_sup}
\begin{center}
\begin{small}
\begin{sc}
\begin{tabular}{lccr}
\toprule
Model & Memory  / Seg Len & PPL \\
\midrule
% Transformer    & 0 / 70 & 56.4 \\
% Transformer-XL & 70 / 70 & 54.5 \\
CRT - GRU & 1 / 70 & 56.5 \\
CRT - NCGRU   & 1 / 70 & 55.7 \\
\midrule
% Transformer    & 0 / 35 & 61.7 \\
% Transformer-XL & 35 / 35 & 57.6 \\
CRT - GRU & 1 / 35 & 58.2 \\
CRT - NCGRU   & 1 / 35 & 58.3 \\
\midrule
% Transformer    & 0 / 17 & 79.6 \\
% Transformer-XL & 17 / 17 & 68.2 \\
CRT - GRU & 1 / 17 & 64.2 \\
CRT - NCGRU   & 1 / 17 & 63.0 \\
\bottomrule
\end{tabular}
\end{sc}
\end{small}
\end{center}
\vskip -0.3in
\end{table}

\subsection{WikiText103}\label{subsec:wiki_sup}

\begin{table}[ht]
\caption{Wiki Text-103 PPL results on 3-layer models using NCGRU and GRU.}
\label{wikitext103three_sup}
\begin{center}
\begin{small}
\begin{sc}
\begin{tabular}{lccr}
\toprule
Model & Memory / Seg Len & PPL \\
\midrule
% Transformer    & 0 / 150  & 39.1 \\
% Transformer-XL & 150 / 150 & 32.6 \\
CRT - GRU   & 1 / 150 & 32.3 \\
CRT - NCGRU   & 1 / 150 & 31.8 \\
\midrule
% Transformer    & 0 / 70  & 47.3 \\
% Transformer-XL & 70 / 70 & 36.5 \\
CRT - GRU   & 1 / 70 & 36.6 \\
CRT - NCGRU    & 1 / 70 & 35.8 \\
\midrule
% Transformer    & 0 / 35  & 62.4 \\
% Transformer-XL & 35 / 35 & 43.9 \\
CRT - GRU   & 1 / 35 & 43.0 \\
CRT - NCGRU    & 1 / 35 & 43.0 \\
\bottomrule
\end{tabular}
\end{sc}
\end{small}
\end{center}
\vskip -0.1 in
\end{table}

Following a similar analysis to Section \ref{subsec:sup_wordptb}, Table \ref{wikitext103three_sup} demonstrates that using a three-layer model with NCGRU in place of GRU at the recurrence yields improved performance for sequence lengths of 150 and 70 on WikiText103. However, for sequence length 35, both NCGRU and GRU produce identical results. Furthermore, as shown in Table \ref{wikitext103sixteen_sup}, the outcomes for NCGRU and GRU are nearly equivalent.

\begin{table}[ht]
\caption{Wiki Text-103 PPL results on 16-layer models using NCGRU and GRU}
\label{wikitext103sixteen_sup}
\begin{center}
\begin{small}
\begin{sc}
\begin{tabular}{lccr}
\toprule
Model & Memory / Seg Len & PPL \\
\midrule
% Transformer    & 0 / 150  & 30.0 \\
% Transformer-XL & 150 / 150 & 24.0 \\
CRT - GRU   & 1 / 150 & 25.4  \\
CRT - NCGRU & 1 / 150 & 25.8 \\
\midrule
% Transformer    & 0 / 70  & 37.4 \\
% Transformer-XL & 70 / 70 & 27.6 \\
CRT - GRU  & 1 / 70 & 28.6  \\
CRT - NCGRU & 1 / 70 & 28.8 \\
\midrule
% Transformer    & 0 / 35  & 51.9 \\
% Transformer-XL & 35 / 35 & 34.2 \\
CRT - GRU  & 1 / 35 & 35.5  \\
CRT - NCGRU & 1 / 35 & 35.8 \\
\bottomrule
\end{tabular}
\end{sc}
\end{small}
\end{center}
\vskip -0.2in
\end{table}

\subsection{Toyota Smarthome}

\begin{table}[ht]
\caption{Mean class accuracy (mCA) reported for cross-subject (CS) evaluation in Toyota Smarthome Dataset using NCGRU and GRU.}
\label{toyota_sup}
\begin{center}
\begin{small}
\begin{sc}
\begin{tabular}{lccr}
\toprule
Model & mCA \\
\midrule
% 2s-AGCN & 60.9 \\
% PoseC3D & 50.6 \\
% Hyperformer & 57.5 \\
% VPN & 65.2 \\
% AssembleNet++ & 63.6 \\
% UNIK & 64.6 \\
% Video Swin & 69.8 \\
% MotionFormer & 65.8 \\
% MMNET & 70.1 \\
CRViT - NCGRU  & 72.2 \\
CRViT - GRU  & 73.4 \\
% \midrule
% TimeSformer    & 68.4 \\
% \hspace{0.2cm} +2D-SIM & 72.5 \\
% \hspace{0.2cm} +3D-SIM & 71.4 \\
% PI-ViT (2D-SIM + 3D-SIM) & 72.9 \\
\bottomrule
\end{tabular}
\end{sc}
\end{small}
\end{center}
\vskip -0.1in
\end{table}

As discussed in Section \ref{subsec:sup_wordptb} and Section \ref{subsec:wiki_sup}, most experiments show that replacing GRU with NCGRU leads to improved results, primarily because the orthogonality in NCGRU aids in retaining more information across the two language datasets. However, in the Toyota Smarthome dataset experiments, where the sequence length (number of frames per video) is set to a short value of 16, NCGRU offers limited benefits. GRU effectively transfers information from each video frame due to the shorter sequence length. This is evident in Table \ref{toyota_sup}, where GRU outperforms NCGRU in the proposed CRViT model, achieving mCA scores of 73.4 and 72.2, respectively. Moreover, as previously mentioned, GRU's performance also sets a state-of-the-art benchmark for the Toyota Smarthome dataset.

\section{Proof of Equation \ref{eq:boundequation} in Section \ref{sec:derivatives}}

In this section, we present the detailed analysis of gradient presented in section \ref{sec:derivatives}.

We consider the GRU architecture  outlined below:

\begin{equation}
    \label{model:GRU}
    \begin{split}
        r_t &= \sigma \left(W_rx_t + U_r h_{t-1} + b_r\right)\\
        u_t &= \sigma \left(W_ux_t + U_u h_{t-1} + b_u\right)\\
        c_t &= \Phi \left(W_cx_t + U_c \left(r_t \odot h_{t-1}\right) + b_c\right)\\
        h_t &= \left(1 - u_t\right) \odot h_{t-1} + u_t \odot c_t
    \end{split}
\end{equation}

We need some properties related to GRU derivatives given in \cite{mucllari2022orthogonal}:

\begin{lemma} \cite{mucllari2022orthogonal}
    \label{lem:ncgrupaper}
    Let $h_{t-1}$ and $h_{t}$ be two consecutive hidden states from the GRU model. Then
    \begin{equation}
    \begin{split}
        \left\|\frac{\partial h_{t}}{\partial h_{t-1}}\right\|_2 \leq \alpha + \beta\|U_c\|_2
    \end{split}
    \end{equation}
    where $\alpha$ and $\beta$ are given as in \cite{mucllari2022orthogonal}:
    \begin{equation}
        \begin{split}
\alpha &= \delta_u\left( \max _i\left\{[h_{t-1}]_i\right\} + \max _i\left\{[c_t]_i\right\}\right)\|U_u\|_2  + \max _i\left\{(1-[u_t]_i)\right\}
        \end{split}
    \end{equation}
    and
    \begin{equation}
        \begin{split}
            \beta &= \max _i\left\{[u_t]_i\right\}\left(\delta_r \|U_r\|_2 \max _i \left\{[h_{t-1}]_i\right\}\right. \left.+ \max _i \left\{[r_t]_i\right\} \right),
        \end{split}
    \end{equation}
    with constants $\delta_u$ and $\delta_r$ defined as follows:
    \begin{equation}
    \begin{split}
        \delta_u = \max_i \left\{\left[u_t\right]_i \left(1-\left[u_t\right]_i\right)\right\}
        \end{split}
    \end{equation}
    and
    \begin{equation}
    \begin{split}
        \delta_r = \max_i \left\{\left[r_t\right]_i \left(1-\left[r_t\right]_i\right)\right\}.
        \end{split}
    \end{equation} 
\end{lemma}

An additional crucial result concerning the derivatives of GRU is the following:

\begin{lemma}\cite{mucllari2022orthogonal}
\label{lem:gruderiv}

For the hyperbolic tangent activation function in (\ref{model:GRU}) (i.e. $\Phi=\,$\verb|tanh|), we have $\delta_u, \delta_r \le \frac{1}{4}$, $[h_t]_i \le 1$ for any $i$ and $t$ as well as 
    \begin{equation}
        \alpha \le \frac{1}{2}\|U_u\|_2+1 \quad\text{and}\quad \beta\leq \frac{1}{4}\|U_r\|_2+1.
    \end{equation}
\end{lemma}

The proof of Lemma \ref{lem:ncgrupaper} and Lemma \ref{lem:gruderiv} can be found in the Appendix in \cite{mucllari2022orthogonal}.
Furthermore, it is shown that in the case of an NCGRU, when $u_t$ and $r_t$ are approximately either the zero vector or the vector of all ones, the following holds:
\begin{equation}
       \alpha + \beta\|U_c\|_2 \lesssim 1.
    \end{equation}
 
We can now advance to the proof of our main theorem.

\begin{theorem}
\label{thm:derivativememory} 
Consider the CRT model with GRU defined in (\ref{model:GRU}).
Let $y_k \in \mathbb{R}^{d_m}$ be the output of the Transformer at iteration $t$, where $tn+1 \leq k \leq (t+1)n$ and $n$ represents the segment length in each iteration, and let $y_i \in \mathbb{R}^{d_m}$ be the output of the  Transformer layer at iteration $p$, where $pn+1 \leq i \leq (p+1)n$. The derivative of $y_k$ with respect to $y_i$ when $t>p$ is expressed as:
\begin{equation}
\label{eq:derivativethmappend}
    \begin{split}
        \frac{\partial y_k}{\partial y_i} & = \frac{\partial y_k}{\partial m_{tn}}  \frac{\partial h_{tn}}{\partial h_{i}}  \frac{\partial h_i}{\partial y_i}
    \end{split}
\end{equation}
Furthermore, 
\begin{equation}
\label{eq:derivativethmappendbound}
    \begin{split}
     \left\| \frac{\partial y_k}{\partial y_i} \right\|_2   & \leq   (\alpha + \beta\|U_c\|_2)^{tn-i}  
     \left\|\frac{\partial y_k}{\partial m_{tn}}  \right\|_2        
     \left\|  \frac{\partial h_j}{\partial y_i} \right\|_2      
    \end{split}
\end{equation}
where $\alpha, \beta$ are defined in Lemma A.1.
% $$\frac{\partial y_k}{\partial y_i} = \frac{\partial y_k}{\partial m_t} \left(\prod _{j=p+1}^{t-1} \frac{\partial m_{j+1}}{\partial m_{j}}\right) \frac{m_p}{\partial y_i}$$

\end{theorem}

\begin{proof}
(\ref{eq:derivativethmappend}) follows from the chain rule and the fact that  $y_k$ is a function of $ y_i$ through the memory token $ m_{tn}$ defined in the $i$th Transformer,   $ m_{tn}=  h_{tn}$ is a function of $h_i$ defined by the GRU architecture, and the GRU state $h_i$ is a function of its input $y_i$. Using the chain rule again, we have 
\[
\frac{\partial h_{tn}}{\partial h_{i}} = \prod _{j=i}^{tn-1} \frac{\partial h_{j+1}}{\partial h_{j}}. 
\]
Bounding $\frac{\partial h_{j+1}}{\partial h_{j}}$ using Lemma A.1, we obtain the desired bound. 
\end{proof}

In particular, in the case that NCGRU is used, when $u_t$ and $r_t$ are approximately either the zero vector or the vector of all ones, the bound becomes:
\begin{equation}
     \left\| \frac{\partial y_k}{\partial y_i} \right\|_2   \& \leq    
     \left\|\frac{\partial y_k}{\partial m_{tn}}  \right\|_2        
     \left\|  \frac{\partial h_j}{\partial y_i} \right\|_2     
    \end{equation}
\end{document}